\documentclass[journal]{IEEEtran}

\usepackage{amsmath,amssymb,amsthm,mathtools}
\usepackage{bm}
\usepackage{booktabs}
\usepackage{graphicx}
\usepackage{tikz}
\usepackage{cite}
\usepackage[hidelinks]{hyperref}
\usepackage[caption=false,font=footnotesize]{subfig}
\usetikzlibrary{arrows.meta,calc,positioning}

\definecolor{figred}{RGB}{238,35,35}
\definecolor{figblue}{RGB}{24,52,235}
\definecolor{figgreen}{RGB}{28,100,39}
\definecolor{figorangefill}{RGB}{250,218,196}
\definecolor{figbluefill}{RGB}{222,231,255}
\definecolor{figgreenfill}{RGB}{229,241,220}

\tikzset{
  panel/.style={rounded corners=1.4pt, line width=.55pt},
  axis/.style={black, line width=.62pt,
    -{Stealth[length=3.0pt,width=2.5pt]}},
  leader/.style={line width=.72pt,
    -{Stealth[length=3.5pt,width=3.0pt]}},
  flow/.style={black, line width=1.8pt,
    -{Stealth[length=5.5pt,width=5.2pt]}},
  title/.style={font=\fontsize{11.4}{12.0}\selectfont\bfseries,
    inner sep=0pt},
  subtitle/.style={font=\fontsize{8.5}{9.2}\selectfont\itshape,
    inner sep=0pt},
  annoLg/.style={font=\fontsize{7.9}{9.0}\selectfont, inner sep=1pt},
  small/.style={font=\fontsize{7.15}{8.1}\selectfont, inner sep=.5pt},
  mathlab/.style={font=\fontsize{8.6}{9.3}\selectfont, inner sep=.5pt},
  labelbox/.style={fill=white, fill opacity=.96, text opacity=1,
    rounded corners=.8pt, inner sep=1.5pt},
}

\hypersetup{
  pdftitle={Certified Safety Radii in Forecast-Error Space for Wasserstein Distributionally Robust Small Signal Stability-Constrained AC Optimal Power Flow via Lifted Spectrahedral Containment}
}

\newtheorem{theorem}{Theorem}
\newtheorem{proposition}{Proposition}
\newtheorem{lemma}{Lemma}

\newif\ifprintreferences
\printreferencestrue

\begin{document}

\title{Certified Safety Radii in Forecast-Error Space
for Wasserstein Distributionally Robust Small Signal Stability-Constrained
AC Optimal Power Flow via Lifted Spectrahedral Containment}


\author{Ziqi Zhang,~\IEEEmembership{Student Member,~IEEE} and
	Xi Chen,~\IEEEmembership{Member,~IEEE}

	\thanks{
		Z. Zhang and X. Chen are with the College of Automation Engineering, Nanjing University of Aeronautics and Astronautics, Nanjing 211100, China
		(E-mail: 1124562662@qq.com, chenxi0514@nuaa.edu.cn). } 
}
\maketitle

\begin{abstract}
Directly robustifying small-signal stability in AC optimal power flow (OPF) is challenging since the stability boundary in the original uncertainty space is implicit, highly nonconvex, and changes with the operating decision. 
This paper exploits an alternative geometry. 
For a fixed model-specific stability certificate admitting suitable physical lifts, the small-signal stability requirement becomes an affine positive-semidefinite (PSD) constraint in the lifted variables, thereby defining a convex certified safe region. 
Instead of approximating the nonlinear instability boundary itself, we optimize a sample-wise safe radius in the original uncertainty space and certify, in the lifted space, that the entire power-flow image of the corresponding uncertainty ball is contained in the convex stability region.
To this end, a componentwise Perron certificate guarantees existence, uniqueness, and Jacobian regularity of the target AC power-flow branch throughout each ball. 
An adjoint elimination then provides an exact affine–quadratic representation of the stability-relevant quantities, while rigorous matrix-remainder bounds convert their nonlinear variation into finite robust PSD constraints. 
The resulting radii are certified lower bounds on the distances from empirical samples to failure and can therefore be coupled directly to the distance-based reformulation of a Wasserstein distributionally robust chance constraint, without directly approximating the instability boundary.  
Numerical studies demonstrate the effectiveness of the proposed framework.
\end{abstract}


\section{Introduction}
\label{sec:introduction}

\IEEEPARstart{F}{or} each empirical forecast-error sample, the exact
sample-distance representation of a 1-Wasserstein distributionally robust
chance constraint asks for one number: how far is the sample from the
closed failure set?\cite{chen2024chance}
In AC optimal power flow (OPF)
with small-signal stability (SSS) requirements, this number is hidden
behind the equilibrium map.  A forecast error has no direct stability
label.  It first perturbs the nodal injections, then induces an equilibrium
on a nonlinear AC power-flow (PF) branch, and only at that equilibrium can
operating limits and the linearized dynamics be assessed.  The first loss
of safety may therefore be the loss of the target PF branch, a singular
PF Jacobian, an operating-limit violation, or a loss of SSS.  Moreover, all
of these failure boundaries move with the dispatch decision.  The
statistical risk model and the physical system thus meet at a specific
missing interface: a certified sample-to-failure distance measured in the
same forecast-error coordinates and ground metric as the data.

The problem contains an exploitable geometric asymmetry.  For the dynamic
model specified in each application, a model-specific SSS theorem can be
written, after suitable physical lifting, as an affine positive-semidefinite
(PSD) matrix inequality in equilibrium-dependent quantities. 
This paper exploits the convex geometry of the lifted certificate without
moving the data or the transportation metric out of the physical
forecast-error space.  Instead of constructing an approximation of the
entire failure boundary, we certify a ball around each empirical sample
whose complete nonlinear PF image remains inside the fixed lifted safety
region.

\subsection{Related Work and Positioning}
\label{subsec:related_work}

Closest-bifurcation and closest-security-boundary methods have established
systematic procedures for computing critical operating points, limiting
directions, and control margins in multidimensional parameter spaces
\cite{dobson1993closest,alvarado1994closest,
	gomes2017closest,li2018margin}.
Related adversarial-distance formulations have also examined the smallest
perturbation that causes an optimization model to fail, including
DC-OPF infeasibility \cite{chevalier2026adversarial}.  This body of work is
organized around locating a critical boundary point and is particularly
valuable for mechanism diagnosis, margin assessment, and corrective-control
design.  The Wasserstein interface considered here requires a different
output: a one-sided lower distance for every empirical forecast-error
sample to the joint loss of the target PF branch, PF regularity, operating
feasibility, or the declared SSS property.  We therefore use independently
computed boundary distances as benchmarks for radius tightness, while the
optimization itself is based on certified interior containment rather than
nearest-boundary search.

A second body of work characterizes SSS under multidimensional
uncertainty.  Small-signal admissible regions, robust SSS regions,
polynomial bifurcation surfaces, and structured stability radii describe
how stability changes over uncertain injection or dynamic-parameter
spaces
\cite{pan2016ssar,pan2018rsssr,shen2022bifurcation,
	jin2026robustregion}.
Probabilistic and risk-based studies instead propagate uncertainty to
critical damping ratios, spectral indices, or instability probabilities
\cite{preece2014probabilistic,preece2015risk}.
These approaches provide region-level or distribution-level descriptions
of uncertain stability.  The object required in the present work is
sample-wise and metric-specific: the reported quantity must be a certified
lower distance in the operational forecast-error coordinates, and its
failure event also includes the nonlinear PF branch and operating limits.
Accuracy of a fitted region and the one-sided validity
of such a distance are therefore distinct properties.

SSS constraints, including formulations under uncertainty, have also been
incorporated directly into economic dispatch
\cite{hamon2013partI, 
	pareek2021convexification,wang2024twostage,yu2025stochasticrobust,chu2025distributional}.
These studies establish the broader class of stability-constrained
optimization, including its uncertainty-aware variants, and employ different
uncertainty objects, stability representations, and solution mechanisms.  The present
paper is positioned within this class.  Its contribution is not the broad
category of stochastic, robust, or distributionally robust SSS-OPF, but
the construction of a certified sample-distance interface for operational
forecast errors whose consequences are evaluated through the nonlinear
AC equilibrium.

The two ends of the required construction are also well developed.
Fixed-point conditions for balanced distribution-network PF models and
convex restrictions for parameterized PF equations provide sufficient
solvability guarantees over prescribed parameter regions
\cite{bolognani2016existence,dvijotham2018solvability,
nguyen2019inner,lee2019convexrestriction}.
Monotonicity-based methods construct voltage domains containing at most
one PF solution \cite{dvijotham2015differential}.
Robust convex restrictions extend these ideas to AC-OPF under uncertain
injections \cite{lee2021robustopf}.  We use these results together with
explicit PF-Jacobian regularity, target-branch propagation, and downstream
safety-matrix containment.  At the statistical end,
Wasserstein distributionally robust optimization provides data-driven
performance guarantees and tractable reformulations
\cite{esfahani2018wasserstein,duan2018wasserstein,chen2024chance}.
The PF literature certifies an equilibrium region,
whereas the Wasserstein risk formulation aggregates supplied distances.
The interface addressed in this paper is the nonlinear pullback between
them: from a fixed convex safety certificate in lifted equilibrium
coordinates, through the target AC-PF branch, to a certified lower
distance in the original forecast-error metric.

\subsection{Research Gap and Contributions}
\label{subsec:contributions}
A key difficulty lies in the interface between small-signal stability
certification and data-driven distributional robustness. Small-signal
stability is evaluated at an AC equilibrium that depends implicitly on
both the dispatch and the forecast error, whereas sample-distance
Wasserstein reformulations require distances to failure directly in the
original uncertainty space. These two descriptions are separated by the
nonlinear, branch-dependent mapping from forecast errors to the
corresponding AC equilibrium and then to stability and operating safety,
leaving the relevant sample-to-failure distances implicit and difficult
to embed directly in OPF without repeatedly solving closest-failure
problems.
Accordingly, the central question addressed in this work is how to obtain
tractable, decision-dependent lower bounds on these distances while
preserving the nonlinear AC power-flow branch and a rigorous
small-signal-stability guarantee. 
For each empirical sample \(\widehat\xi_i\), we therefore seek a
decision-dependent certified radius \(r_i\) in the original forecast-error
space.  The corresponding weighted-\(\ell_1\) neighborhood is required to
remain entirely within the certified safe region, which itself is contained
in the target-model safe region.  Consequently, \(r_i\) is a rigorous lower
bound on both the certified and physical sample-to-failure distances.  These
sample-wise radii provide the distance information required by the subsequent
Wasserstein risk model without explicitly constructing the nonlinear failure
boundary or repeatedly solving closest-failure problems.

The main contributions are summarized as follows.
\begin{itemize}
	
	\item
	\emph{Certified sample-wise stability distances:}
	We construct decision-dependent safety radii \(r_i\) in the original
	forecast-error space that rigorously lower-bound the implicit
	sample-to-failure distances associated with the target PF branch, PF
	regularity, operating limits, and the selected small-signal-stability
	certificate.
	
	\item
	\emph{Nonlinear PF-to-PSD containment:}
	We develop a self-consistent component-Perron PF tube and combine it with
	an exact adjoint elimination and matrix-level quadratic-remainder
	certificates, preserving safety-output cancellations while reducing
	whole-sample-ball safety to finite SDP/SOCP constraints.
	
	\item
	\emph{End-to-end stability-certified WDRO-OPF:}
	We couple these decision-dependent certified radii with the existing exact
	sample-distance Wasserstein reformulation and jointly optimize them with
	the dispatch in a sequential conic master, avoiding embedded
	closest-failure computations.
	
\end{itemize}

The remainder of this paper is organized as follows.
Section~\ref{sec:risk_interface} defines the fixed lifted-PSD certificate,
the target-model safe sets, and the sample-distance Wasserstein risk
interface.
Section~\ref{sec:certified_radii} develops the regular-root PF tube, the exact
adjoint safety-matrix identity, and the finite robust containment conditions.
Section~\ref{sec:wasserstein_master} presents the sequential
stability-certified WDRO-OPF, the propagation of the target-branch label
between accepted iterates, and the distributional guarantee.
Section~\ref{sec:case_studies} reports the numerical studies, and
Section~\ref{sec:conclusion} concludes the paper.

\section{Problem Formulation and Distance-Based Risk}
\label{sec:risk_interface}

\subsection{Fixed Lifted-PSD Stability Certificate}

The dispatch vector \(u\in\mathcal U\subseteq\mathbb R^{n_u}\) contains the scheduled generation,
renewable curtailment, reserves, and continuous control setpoints.  The
uncertainty vector \(\xi\in\mathbb R^m\) records realized-minus-forecast
errors in available renewable power and in active and reactive demand.
These physical forecast errors remain the coordinates of the statistical
model.  Fixed allocation and participation factors map them affinely to
nodal injections.
The steady-state vector \(x\in\mathbb R^{n_x}\) contains the rectangular bus voltages
\(V_b=e_b+\mathrm{j}f_b\) and the balancing and voltage-magnitude
auxiliaries required by the operating model.
We write its equilibrium equations as \(F(x,u,\xi)=0\), where
\(F:\mathbb R^{n_x}\times\mathbb R^{n_u}\times
\mathbb R^m\to\mathbb R^{n_x}\).
Specifically, \(F\) stacks active- and reactive-power balance at the buses,
the reference-angle and voltage-control equations, the distributed
balancing equation that accounts for AC losses, and any exact
voltage-magnitude identities used by the device model.  Certificate-specific
equilibrium coordinates, when required, are appended to \(x\) together with
their exact graph equations in \(F\).  The construction below requires the
resulting equilibrium equations and safety outputs to retain the stated
affine--quadratic form.  A verified exact root at a base equilibrium
and a continuation rule identify the target solution branch,
denoted by \(x_{\rm PF}(u,\xi)\); \(J_xF\) is the power-flow (PF) Jacobian
along that branch. 
Small-signal stability is evaluated for a specified dynamic model.
Its controller orders, network and load dynamics, algebraic variables,
parameter range, and symmetry-reduced perturbation subspace define the model
assumptions. 
Let \(z=h(x,u,\xi)\in\mathbb R^{n_z}\) collect the equilibrium quantities
entering the fixed stability certificate and the operating-limit blocks.
Before solving the OPF, we fix one
complete certificate with the affine symmetric pencil
\begin{equation}
K^{\rm case}(z,u)
=K_0+K_u(u)+\sum_{a=1}^{n_z}z_aK_a ,
\label{eq:affine_pencil}
\end{equation}
where \(K_u\) is affine and \(K_a\) are fixed symmetric matrices.  The
online stability requirement is \(K^{\rm case}(z,u)\succeq\tau_KI\), with
a prescribed margin \(\tau_K>0\).  Its feasible set in \((z,u)\) is a
spectrahedron and is therefore convex.

Several model-specific small-signal stability results admit the
affine-PSD interface in \eqref{eq:affine_pencil}. Examples include
projected network--device stability matrices for lossless grid-forming
(GFM) systems \cite{bacic2026graph}. For grid-following (GFL) systems, gSCR/gOSCR conditions
become PSD network-strength inequalities once the model-specific critical
threshold has been certified over the operating domain
\cite{dong2019gscr,liu2024goscr}. For heterogeneous GFM/GFL/HVDC and
dynamic-load subsystems, passivity and dissipativity conditions yield KYP
or descriptor-KYP inequalities after the storage matrices, supply rates,
and multipliers have been fixed offline
\cite{rantzer1996kyp,camlibel2009descriptorkyp}.
When a model-specific stability result requires multiple online matrix
conditions, they are imposed jointly.
Sections~\ref{sec:certified_radii} and \ref{sec:wasserstein_master} are
stated for this general interface.  
This study instantiates it
with the Iva certificate of \cite{bacic2026graph}.  This specialization
assumes a fixed, connected, lossless effective network with
\(B_{bc}=B_{cb}\ge0\) for \(b\ne c\), under the susceptance and shunt sign
conventions of that reference.  Loads are fixed-power static loads, with
instantaneous load-only buses eliminated by a valid Kron reduction.  Each
active node uses standard \(q\)--\(V\) droop with fixed \(\beta_b^q>0\) in
an unsaturated operating mode.  Conditions~1 of \cite{bacic2026graph} hold
uniformly on the certified domain: the modified device-response inverses
are real-rational and pole-free on the closed right half-plane, their
Hermitian parts are positive definite there, and they satisfy a common
high-frequency coercivity bound.  The effective node set and an
orthonormal basis \(O_\perp\) of the uniform-angle complement are fixed.
For this specialization, introduce the positive equilibrium-voltage
coordinate \(\upsilon_b:=|V_b|>0\) and append the exact identity
\(\upsilon_b^2=e_b^2+f_b^2\) to the lossless AC equations.  Thus,
\(x_{\rm Iva}=\operatorname{col}(e,f,\upsilon,\ldots)\), while
\(F_{\rm Iva}\) stacks \(F_{\rm AC}^{\rm lossless}(e,f,u,\xi)\) and these
magnitude identities.
The certified domain enforces \(\upsilon_b\ge\underline v_b>0\), thereby
selecting the physical magnitude branch.  Define
\(c_{bc}=e_be_c+f_bf_c\),
\(\sigma_{bc}=f_be_c-e_bf_c\), and \(w_b=e_b^2+f_b^2\), and collect
\(z_{\rm Iva}=\operatorname{col}(\{c_{bc}\},\{\sigma_{bc}\},
\{w_b\},\{\upsilon_b\})\).  Standard \(q\)--\(V\) droop gives
\(k_b^q=\beta_b^q/\upsilon_b\).  The complete projected pencil is therefore
\begin{equation}
	\begin{aligned}
		K_{\rm Iva}(z_{\rm Iva})
		&=O_\perp^{\mathsf T}\Xi_{\rm Iva}(z_{\rm Iva})O_\perp,\\
		\Xi_{\rm Iva}(z_{\rm Iva})
		&=M_{\rm net}(c,\sigma,w)+
		\begin{bmatrix}
			0&0\\
			0&\operatorname{diag}_b(\upsilon_b/\beta_b^q)
		\end{bmatrix}.
	\end{aligned}
	\label{eq:iva_pencil}
\end{equation}
Here, \(M_{\rm net}\) is the steady-state network Jacobian in angle and log-voltage
coordinates.  Its entries are affine in \((c,\sigma,w)\), while the
inverse-droop block is affine in \(\upsilon\) because \(\beta^q\) is fixed.
Consequently, \(K_{\rm Iva}\) is affine in \(z_{\rm Iva}\), and
\(\{z_{\rm Iva}:K_{\rm Iva}(z_{\rm Iva})\succeq\tau_KI\}\) is a
spectrahedron.  The full lift used below is
\(z=\operatorname{col}(z_{\rm Iva},z_{\rm op})\), where \(z_{\rm op}\)
contains any additional affine--quadratic outputs required by the
operating-limit blocks and may be empty; \(K_{\rm Iva}\) has zero
coefficients on these coordinates.  The magnitude identities in
the augmented equilibrium equations remain part of the nonlinear physical graph;
they are not relaxed by this lifted geometry.  Under
these Iva-specific assumptions, \cite[Th.~1]{bacic2026graph} makes
\(K_{\rm Iva}\succ0\) necessary and sufficient for asymptotic stability of
the specified linearized closed loop after removal of the uniform-angle
mode.  The buffered condition \(K_{\rm Iva}\succeq\tau_KI\),
\(\tau_K>0\), is the inner certificate used here.
The central difficulty remains its physical preimage
\(\xi\mapsto x_{\rm PF}(u,\xi)\mapsto
h(x_{\rm PF}(u,\xi),u,\xi)\), which is nonlinear and defined through the
selected PF branch.

\subsection{Certified Safe Set and Sample-Safe Radius}

Let \(g_\ell(x,u,\xi)<0\), \(\ell=1,\ldots,n_g\), collect the voltage,
generation, reserve, and line-flow limits.  Let
\(\alpha_\perp(A_{\rm dyn}^{\rm tgt})\) denote the spectral abscissa of the
target-model linearization.  Here, \(g_\ell\),
\(A_{\rm dyn}^{\rm tgt}\), and \(J_xF\) are evaluated at
\(x=x_{\rm PF}(u,\xi)\).  For a fixed dispatch, the physical safe set is
\begin{equation}
\begin{split}
\mathcal S^{\rm phys}(u):=\{\xi:\;
x_{\rm PF}(u,\xi)
\text{ exists on the target branch},\\
 J_xF\text{ is nonsingular},\quad g_\ell<0\ \forall\ell,\quad
 \alpha_\perp(A_{\rm dyn}^{\rm tgt})<0\}.
\end{split}
\label{eq:physical_safe_set}
\end{equation}
The superscript \({\rm phys}\) refers throughout to this specified target
model.  A higher-fidelity interpretation requires a separate uniform bridge
covering its equilibria, regularity, limits, and dynamics. 
The fixed pencil defines the certified safe set
\begin{equation}
\begin{split}
\mathcal S^{\rm cert}(u):=\{\xi:\;&x_{\rm PF}(u,\xi)
\text{ exists on the target branch},\\
&J_xF\text{ is nonsingular},\quad g_\ell<0\ \ \forall\ell,\\
&K^{\rm case}(h(x_{\rm PF}(u,\xi),u,\xi),u)\succ0\}.
\end{split}
\label{eq:certificate_safe_set}
\end{equation}
The certificate theorem gives
\(\mathcal S^{\rm cert}(u)\subseteq\mathcal S^{\rm phys}(u)\). 
Let \(D\in\mathbb R^{m\times m}\) be a nonsingular metric matrix fitted independently of the OPF
samples and set \(\|\delta\|_D=\|D\delta\|_1\).  The data and transport
remain in forecast-error space; \(D\) only sets their scale and directional
cost.  The strict conditions in \eqref{eq:certificate_safe_set} make
\(\mathcal S^{\rm cert}(u)\) open on the selected branch.  Its failure set
\(\mathcal F^{\rm cert}(u)=
\mathbb R^m\setminus\mathcal S^{\rm cert}(u)\) is therefore closed.
For sample \(\widehat\xi_i\), define
\begin{equation}
d_i^{\rm cert}(u)
:=\inf_{\xi\in\mathcal F^{\rm cert}(u)}
\|D(\xi-\widehat\xi_i)\|_1 .
\label{eq:cert_distance}
\end{equation}
This distance reaches the first loss of the target PF
branch, Jacobian regularity, an operating limit, or certified
small-signal stability.  Computing it directly entails a nonconvex search
over an implicitly defined failure boundary.
We instead optimize a radius \(r_i\ge0\) by certifying the set inclusion
\begin{equation}
\mathcal B_i(r_i):=
\{\widehat\xi_i+\delta:\|D\delta\|_1\le r_i\}
\subseteq\mathcal S^{\rm cert}(u).
\label{eq:sample_ball_inclusion}
\end{equation}
Denote \(d_i^{\rm phys}(u)\) as the analogous distance to
\(\mathbb R^m\setminus\mathcal S^{\rm phys}(u)\), then
$
\mathcal B_i(r_i)\subseteq\mathcal S^{\rm cert}(u)
\,\Longrightarrow\,
0\le r_i\le d_i^{\rm cert}(u)\le d_i^{\rm phys}(u).
$

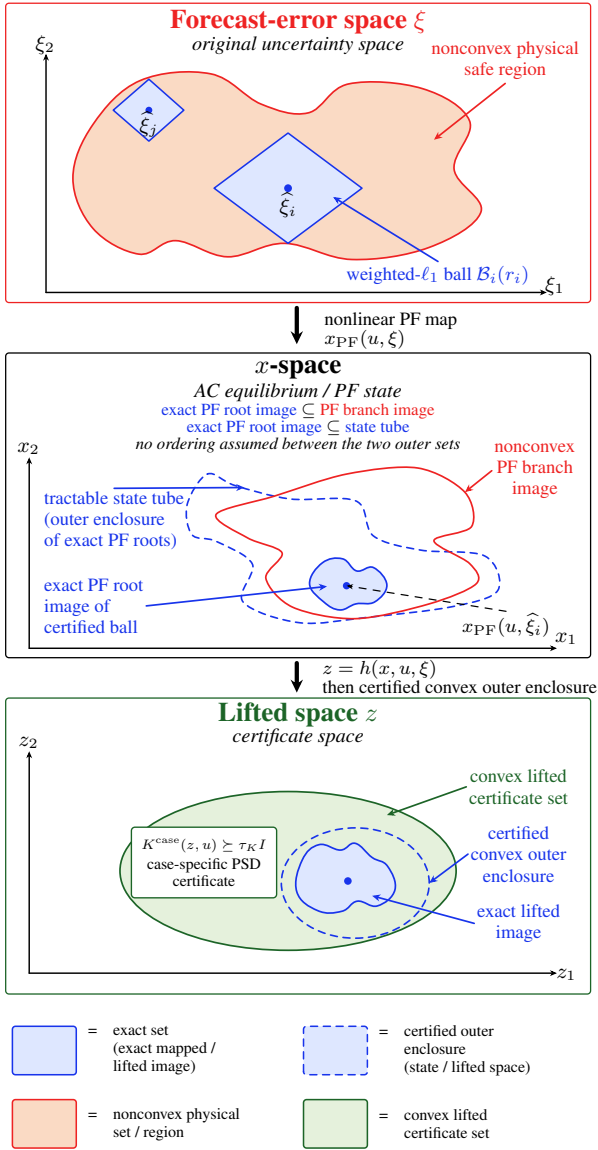
\begin{figure}[ht]
	\centering
	\resizebox{0.9\columnwidth}{!}{%
		\begin{tikzpicture}[x=1cm,y=1cm,line cap=round,line join=round]
			\draw[panel,figred] (0.05,12.90) rectangle (8.75,17.35);
			\node[title,text=figred] at (4.40,17.08) {Forecast-error space $\xi$};
			\node[subtitle] at (4.40,16.73) {original uncertainty space};
			
			\draw[axis] (0.65,13.05) -- (8.10,13.05);
			\draw[axis] (0.65,13.05) -- (0.65,16.60);
			\node[mathlab] at (8.22,13.15) {$\xi_1$};
			\node[mathlab] at (0.64,16.74) {$\xi_2$};
			
			\path[draw=figred,fill=figorangefill,line width=.72pt]
			plot[smooth cycle,tension=.63] coordinates {
				(1.35,14.00) (1.05,14.75) (1.35,15.55) (2.05,16.20)
				(3.10,16.30) (3.85,15.85) (4.55,15.72) (5.15,16.15)
				(6.25,16.08) (6.95,15.60) (6.72,14.88) (7.12,14.25)
				(6.35,13.72) (5.20,13.54) (4.20,13.78) (3.35,13.58)
				(2.58,13.82) (2.03,13.62)
			};
			
			\path[draw=figblue,fill=figbluefill,line width=.66pt]
			(1.65,15.76) -- (2.18,16.22) -- (2.70,15.76) -- (2.18,15.30) -- cycle;
			\fill[figblue] (2.18,15.76) circle (1.45pt);
			\node[mathlab] at (2.18,15.55) {$\widehat\xi_j$};
			
			\path[draw=figblue,fill=figbluefill,line width=.66pt]
			(3.15,14.60) -- (4.25,15.42) -- (5.35,14.60) -- (4.25,13.78) -- cycle;
			\fill[figblue] (4.25,14.60) circle (1.65pt);
			\node[mathlab] at (4.25,14.35) {$\widehat\xi_i$};
			
			\node[annoLg,labelbox,text=figred,align=center] (safe) at (7.48,16.5)
			{nonconvex physical\\safe region};
			\draw[leader,figred] (safe.south) -- (6.42,15.32);
			
			\node[annoLg,labelbox,text=figblue,align=center] (ball) at (6.48,13.31)
			{weighted-$\ell_1$ ball $\mathcal B_i(r_i)$};
			\draw[leader,figblue] (ball.north) -- (4.92,14.47);
			
			\draw[flow] (4.38,12.82) -- (4.38,12.28);
			\node[annoLg,anchor=west,align=left] at (4.75,12.48)
			{nonlinear PF map\\[-1pt]$x_{\rm PF}(u,\xi)$};
			
			\draw[panel,black] (0.05,7.62) rectangle (8.75,12.15);
			\node[title] at (4.40,11.90) {$x$-space};
			\node[subtitle] at (4.40,11.57) {AC equilibrium / PF state};
			
			\node[small] at (4.40,11.29)
			{\textcolor{figblue}{exact PF root image} $\subseteq$ \textcolor{figred}{PF branch image}};
			\node[small] at (4.40,11.04)
			{\textcolor{figblue}{exact PF root image} $\subseteq$ \textcolor{figblue}{state tube}};
			\node[small,font=\fontsize{7.15}{8.1}\selectfont\itshape] at (4.40,10.79)
			{no ordering assumed between the two outer sets};
			
			\draw[axis] (0.40,7.76) -- (8.25,7.76);
			\draw[axis] (0.40,7.76) -- (0.40,10.60);
			\node[mathlab] at (8.39,7.91) {$x_1$};
			\node[mathlab] at (0.39,10.73) {$x_2$};
			
			\path[draw=figblue,dashed,dash pattern=on 3.0pt off 2.4pt,line width=.70pt]
			plot[smooth cycle,tension=.58] coordinates {
				(2.75,9.90) (3.05,10.35) (3.46,10.18) (4.28,9.98)
				(5.26,9.95) (5.52,9.43) (6.22,9.22) (7.18,9.13)
				(7.37,8.62) (6.70,8.38) (5.43,8.22) (4.42,8.14)
				(3.58,8.45) (3.72,9.01) (3.18,9.35) (2.82,9.45)
			};
			
			\path[draw=figred,line width=.72pt]
			plot[smooth cycle,tension=.55] coordinates {
				(3.18,9.75) (4.04,10.02) (5.12,10.35) (6.18,10.44)
				(6.88,10.08) (6.98,9.57) (6.42,9.36) (6.95,9.02)
				(7.05,8.72) (6.35,8.42) (5.25,8.20) (4.32,8.37)
				(3.92,8.80) (4.16,9.18) (3.65,9.40)
			};
			
			\path[draw=figblue,fill=figbluefill,line width=.72pt]
			plot[smooth cycle,tension=.70] coordinates {
				(4.57,8.72) (4.76,9.03) (5.06,9.09) (5.27,8.91)
				(5.52,8.96) (5.72,8.70) (5.54,8.45) (5.26,8.42)
				(5.02,8.33) (4.78,8.45)
			};
			\fill[figblue] (5.12,8.69) circle (1.55pt);
			
			\node[annoLg,text=figblue,align=left,anchor=west] (tube) at (0.58,9.66)
			{tractable state tube\\(outer enclosure\\of exact PF roots)};
			\draw[leader,figblue] (tube.north) -- (3.6,10.15);
			
			\node[annoLg,text=figred,align=center,anchor=west] (branch) at (7.28,10.4)
			{nonconvex\\PF branch\\image};
			\draw[leader,figred] (branch.west) -- (6.88,10.08);
			
			\node[annoLg,text=figblue,align=left,anchor=west] (exactpf) at (0.58,8.39)
			{exact PF root\\image of\\certified ball};
			\draw[leader,figblue] (exactpf.east) -- (4.83,8.68);
			
			\node[mathlab,font=\fontsize{8.2}{8.9}\selectfont,anchor=east]
			(nominal) at (8.17,8.11)
			{$x_{\rm PF}(u,\widehat\xi_i)$};
			\draw[black,dashed,line width=.52pt,-{Stealth[length=2.8pt,width=2.2pt]}]
			(nominal.north) -- (5.12,8.69);
			
			\draw[flow] (4.38,7.54) -- (4.38,7.10);
			\node[annoLg,anchor=west,align=left] at (4.72,7.34)
			{$z=h(x,u,\xi)$\\[-1pt]then certified convex outer enclosure};
			
			\draw[panel,figgreen] (0.05,2.62) rectangle (8.75,7.02);
			\node[title,text=figgreen] at (4.40,6.77) {Lifted space $z$};
			\node[subtitle] at (4.40,6.44) {certificate space};
			
			\draw[axis] (0.40,2.93) -- (8.25,2.93);
			\draw[axis] (0.40,2.93) -- (0.40,6.23);
			\node[mathlab] at (8.39,2.91) {$z_1$};
			\node[mathlab] at (0.39,6.36) {$z_2$};
			
			\draw[figgreen,fill=figgreenfill,line width=.72pt]
			(4.25,4.45) ellipse [x radius=2.50,y radius=1.18];
			
			\node[annoLg,text=figgreen,align=center] (certset) at (7.68,5.70)
			{convex lifted\\certificate set};
			\draw[leader,figgreen] (certset.west) -- (5.72,5.28);
			
			\draw[figgreen,fill=white,rounded corners=1.5pt,line width=.58pt]
			(1.92,4.05) rectangle (4.08,5.10);
			\node[small,font=\fontsize{6.7}{7.2}\selectfont,align=center] at (3.00,4.575)
			{\scalebox{.83}{$K^{\rm case}(z,u)\succeq \tau_K I$}\\[1.0pt]
				case-specific PSD\\certificate};
			
			\draw[figblue,dashed,dash pattern=on 3.0pt off 2.4pt,line width=.70pt]
			(5.25,4.27) ellipse [x radius=1.10,y radius=.82];
			
			\path[draw=figblue,fill=figbluefill,line width=.72pt]
			plot[smooth cycle,tension=.72] coordinates {
				(4.38,4.14) (4.46,4.61) (4.71,4.84) (5.03,4.76)
				(5.25,4.86) (5.50,4.65) (5.73,4.52) (5.84,4.17)
				(5.58,3.93) (5.28,4.00) (5.05,3.85) (4.69,3.86)
			};
			\fill[figblue] (5.14,4.30) circle (1.60pt);
			
			\node[annoLg,text=figblue,align=center] (outerz) at (7.68,4.70)
			{certified\\convex outer\\enclosure};
			\draw[leader,figblue] (outerz.west) -- (6.35,4.27);
			
			\node[annoLg,text=figblue,align=center] (exactz) at (7.68,3.67)
			{exact lifted\\image};
			\draw[leader,figblue] (exactz.west) -- (5.50,4.08);
			
			\draw[figblue,fill=figbluefill,rounded corners=1.3pt,line width=.65pt]
			(0.15,1.42) rectangle (1.10,2.14);
			\node[small,font=\fontsize{6.8}{7.6}\selectfont,anchor=west,align=left]
			at (1.25,1.78)
			{=\quad exact set\\\phantom{=\quad}(exact mapped /\\\phantom{=\quad}lifted image)};
			
			\draw[figblue,fill=figbluefill,dashed,dash pattern=on 3.0pt off 2.4pt,
			rounded corners=1.3pt,line width=.65pt]
			(4.48,1.42) rectangle (5.43,2.14);
			\node[small,font=\fontsize{6.8}{7.6}\selectfont,anchor=west,align=left]
			at (5.58,1.78)
			{=\quad certified outer\\\phantom{=\quad}enclosure\\\phantom{=\quad}(state / lifted space)};
			
			\draw[figred,fill=figorangefill,rounded corners=1.3pt,line width=.65pt]
			(0.15,0.33) rectangle (1.10,1.05);
			\node[small,font=\fontsize{6.8}{7.6}\selectfont,anchor=west,align=left] at (1.25,0.69)
			{=\quad nonconvex physical\\\phantom{=\quad}set / region};
			
			\draw[figgreen,fill=figgreenfill,rounded corners=1.3pt,line width=.65pt]
			(4.48,0.33) rectangle (5.43,1.05);
			\node[small,font=\fontsize{6.8}{7.6}\selectfont,anchor=west,align=left] at (5.58,0.69)
			{=\quad convex lifted\\\phantom{=\quad}certificate set};
		\end{tikzpicture}
	}
	\caption{Geometry of a certified sample-safe radius.  A weighted
		\(\ell_1\) ball \(\mathcal B_i(r_i)\), centered at the empirical
		forecast-error sample \(\widehat\xi_i\), is propagated through the
		nonlinear target-branch PF map.  Its exact root image (solid blue) is
		contained in a tractable state tube (dashed blue); the tube and the full
		nonconvex PF-branch image (red) need not contain one another.  The exact
		lifted image is enclosed by a certified convex set that is required to lie
		inside the model-specific spectrahedron
		\(K^{\rm case}(z,u)\succeq\tau_KI\) (green).  With the operating-limit
		blocks treated in the same way, this proves
		\(\mathcal B_i(r_i)\subseteq\mathcal S^{\rm cert}(u)
		\subseteq\mathcal S^{\rm phys}(u)\) and hence
		\(0\le r_i\le d_i^{\rm cert}(u)\le d_i^{\rm phys}(u)\).}
	\label{fig:sample_safe_geometry}
\end{figure}
Fig.~\ref{fig:sample_safe_geometry} illustrates this inclusion across the
forecast-error, PF-state, and lifted spaces.

\subsection{Exact Wasserstein Risk Aggregation}

For realized-minus-forecast
error samples \(\widehat\xi_1,\ldots,\widehat\xi_N\), let
\(\widehat{\mathbb P}_N=N^{-1}\sum_{i=1}^N
\delta_{\widehat\xi_i}\) be the empirical distribution.  Using
\(\|\cdot\|_D\) as the transportation cost, consider the 1-Wasserstein
ambiguity set
$
\mathbb B_\rho(\widehat{\mathbb P}_N)
:=\{\mathbb Q:W_1^D(\mathbb Q,\widehat{\mathbb P}_N)\le\rho\}.
$
Such empirical Wasserstein sets admit finite-sample coverage guarantees
under standard tail assumptions \cite{esfahani2018wasserstein}.  Because
\(\mathcal S^{\rm cert}(u)\subseteq\mathcal S^{\rm phys}(u)\), controlling
certificate failure also controls target-model failure.
For \(\rho>0\), \(\alpha\in(0,1)\), and the closed set
\(\mathcal F^{\rm cert}(u)\), the sample-distance reformulation of
\cite{chen2024chance} gives
\begin{equation}
\sup_{\mathbb Q\in\mathbb B_\rho(\widehat{\mathbb P}_N)}
\mathbb Q[\xi\in\mathcal F^{\rm cert}(u)]\le\alpha
	\label{eq:dr_chance}
\end{equation}
if and only if there exist \(t\ge0\) and \(s_i\ge0\) such that
$
d_i^{\rm cert}(u) \ge t-s_i,\, 
\rho+\frac{1}{N}\sum_{i=1}^Ns_i \le\alpha t, \,\, i=1,\ldots,N.
$
The certified distance ordering allows the exact distances
to be replaced safely by
$
r_i \ge t-s_i,\, s_i\ge0, 
\rho+\frac{1}{N}\sum_{i=1}^Ns_i \le\alpha t,\, t\ge0 
\,\, i=1,\ldots,N.
$
Conditional on the true distances, the reformulation 
aggregates the Wasserstein risk exactly.  The formulation's conservatism
has two sources: the certified lower bounds
\(r_i\le d_i^{\rm cert}(u)\), and any gap between the fixed certificate and
the target-model stability condition.   
Let $[N]:=\{1,\ldots,N\}$, and let \(C(u)\) denote the operating cost.
The resulting Wasserstein distributionally
robust OPF (WDRO-OPF) has the following oracle form:
\begin{equation}
	\begin{aligned}
		\min_{u,r,t,s}\quad & C(u)\\
		\mathrm{s.t.}\quad
		& \xi^{\rm nom}\in\mathcal S^{\rm cert}(u),\quad
		\mathcal B_i(r_i)\subseteq\mathcal S^{\rm cert}(u),\ \forall i\in[N],\\
		& r_i\ge t-s_i,\quad r_i,s_i\ge0,\ \forall i\in[N],\\
		& \rho+\frac{1}{N}\sum_{i=1}^N s_i\le\alpha t,\quad
		u\in\mathcal U,\quad t\ge0 .
	\end{aligned}
	\tag{P-oracle}\label{eq:oracle_problem}
\end{equation}
The nominal condition in \eqref{eq:oracle_problem} enforces certified
feasibility at the prescribed forecast, for which \(\xi^{\rm nom}=0\) under
the usual error convention.  The sample blocks control distributional risk.
 
\section{Certified Sample-Safe Radii}
\label{sec:certified_radii}

Section~\ref{sec:risk_interface} reduces the distributionally robust
chance constraint to a sample-wise geometric requirement: for each empirical
sample \(\widehat\xi_i\), we must find a radius \(r_i\) such that every
forecast-error realization in
$
\mathcal B_i(r_i)
=
\left\{
\widehat\xi_i+\delta:
\|D\delta\|_1\le r_i
\right\}
$
remains inside the certified safe set.
The difficulty is that safety cannot
be checked directly in the forecast-error coordinates.  Each realization
first determines an equilibrium through the nonlinear and branch-dependent
AC power-flow equations, and the operating and small-signal stability
conditions must then be verified at that equilibrium.
The required inclusion is established through
\begin{equation}
	\begin{aligned}
		\mathcal B_i(r_i)
		&\xrightarrow{\rm PF}\bar x_i+\mathcal Y_i(\beta_i)\\
		&\xrightarrow{\rm adjoint}
		\{\mathsf M_\kappa\}_{\kappa\in\mathcal K_{\rm saf}}
		\xrightarrow{\rm PSD}\mathcal S^{\rm cert}(u).
	\end{aligned}
	\label{eq:radius_certificate_chain}
\end{equation}
Here, \(\bar x_i+\mathcal Y_i(\beta_i)\) is a certified state tube containing
all target-branch equilibria generated by the sample ball, while
\(\{\mathsf M_\kappa\}_{\kappa\in\mathcal K_{\rm saf}}\) denotes the
stability and operating safety-matrix blocks evaluated over that tube.
The three arrows in \eqref{eq:radius_certificate_chain} mark successive
certificates applied to the same physical PF graph.
The first
establishes existence, uniqueness, and Jacobian regularity of the target
PF branch throughout the uncertainty ball.
The second uses exact adjoint
identities to expose the dependence of every safety block on the dispatch,
forecast error, and a quadratic state remainder.
The third bounds these
remainders at the matrix level and enforces robust PSD containment over the
entire tube.
Together, these three steps give
$
	\mathcal B_i(r_i)
	\subseteq
	\mathcal S^{\rm cert}(u)
	\subseteq
	\mathcal S^{\rm phys}(u),
$
and therefore
$
	0\le r_i
	\le d_i^{\rm cert}(u)
	\le d_i^{\rm phys}(u).
$
Accordingly, this section constructs a finite conic certificate for a
sample-centered ball contained in the nonlinear safe set.
For each sample \(i\), the certificate is constructed around a PF anchor
\((\bar x_i,\bar u,\widehat\xi_i)\) associated with the currently accepted
dispatch \(\bar u\).  All local sensitivity coefficients, quadratic
majorants, Perron scalings, adjoint quantities, and matrix-remainder
bounds are then held fixed while the conic master updates \(u\), \(r_i\),
and the tube variables introduced below.  The sequential-iteration
superscript is suppressed for clarity.  The derivation uses exact
coefficients.  In the implementation, sparse linear solves and subsequent
enclosures are verified with outward rounding \cite{rump2010verification};
their one-sided numerical errors are absorbed into the fixed majorants and
matrix tails.
 
\subsection{Branch-Preserving Component-Perron PF Tube}
\label{subsec:pf_tube}

Set \(y=x-\bar x_i\), \(\Delta u=u-\bar u\), and
\(\delta=\xi-\widehat\xi_i\).  Under the rectangular-coordinate
affine--quadratic contract of Section~\ref{sec:risk_interface}, including
any case-specific equilibrium augmentation, the AC equations are quadratic
in the state and affine in \((u,\xi)\); hence
\begin{equation}
	F(\bar x_i+y,u,\widehat\xi_i+\delta)
	=F_i^0+J_i y+B_i\Delta u+C_i\delta+Q_i(y),
	\label{eq:pf_exact_quadratic}
\end{equation}
Here,
\(F_i^0:=F(\bar x_i,\bar u,\widehat\xi_i)\),
\(J_i:=\nabla_xF(\bar x_i,\bar u,\widehat\xi_i)\),
\(B_i:=\nabla_uF(\bar x_i,\bar u,\widehat\xi_i)\), and
\(C_i:=\nabla_\xi F(\bar x_i,\bar u,\widehat\xi_i)\).  Since \(F\) is
quadratic in \(x\),
\([Q_i(y)]_\ell=y^{\mathsf T}H_{i\ell}^{\rm s}y\), where
\(H_{i\ell}^{\rm s}:=\frac12\nabla_{xx}^2
F_\ell(\bar x_i,\bar u,\widehat\xi_i)\).  Equation
\eqref{eq:pf_exact_quadratic} is an exact algebraic identity, not a truncated
Taylor model.  The residual \(F_i^0\) is retained so that a floating-point PF
center is not treated as an exact root.

Let \(M_i=J_i^{-1}\), and let \(m_{ip}^{\mathsf T}\) denote its \(p\)-th
row.  The PF equation is equivalent to the fixed-point problem
\begin{equation}
	y=\Phi_i(y;u,\delta)
	:=-M_i\bigl(F_i^0+B_i\Delta u+C_i\delta+Q_i(y)\bigr).
	\label{eq:pf_fixed_point}
\end{equation}
Fixed-point and convex-restriction methods provide sufficient PF
solvability guarantees
\cite{bolognani2016existence,dvijotham2018solvability,
nguyen2019inner,lee2019convexrestriction}, while monotonicity-based
voltage domains certify that the domain contains at most one PF solution
\cite{dvijotham2015differential}.  Robust convex restrictions were
subsequently extended to AC-OPF under uncertain injections
\cite{lee2021robustopf}.  Here, the componentwise self-mapping and
Perron-scaled contraction conditions are imposed uniformly over each
sample-centered forecast-error ball.  They certify a unique PF root and
a nonsingular PF Jacobian within the tube; separate center-root and
continuation conditions associate this root with the target PF branch
before its matrix-valued safety image is enclosed.
Choose a fixed positive scaling \(\varsigma_i\in\mathbb R_{++}^{n_x}\),
let \(S_i=\operatorname{Diag}(\varsigma_i)\), and define the component tube
\begin{subequations}\label{eq:component_tube}
	\begin{align}
		\mathcal Y_i(\beta_i)
		&:=\{y:\ |y_p|\le \varsigma_{ip}\beta_{ip},\ p=1,\ldots,n_x\},
		\label{eq:component_tube_set}\\
		\beta_{ip}&\ge0,\qquad \chi_{ip}\ge \beta_{ip}^2,
		\quad p=1,\ldots,n_x .
		\label{eq:component_square_epigraph}
	\end{align}
\end{subequations}
Route-specific applicability restrictions are imposed as additional tube or
safety-block conditions.  For the Iva specialization only,
the augmented magnitude equations are included in \(F\), and the component tube also
satisfies
\(\bar\upsilon_{ib}-\varsigma_{i,\upsilon_b}
\beta_{i,\upsilon_b}\ge\underline v_b>0\) at every bus \(b\), so the
entire tube remains on the positive magnitude branch; here
\(\bar\upsilon_{ib}\) is the voltage-magnitude coordinate of the sample
anchor \(\bar x_i\).
The inequalities in \eqref{eq:component_square_epigraph} are rotated-SOC
representable.  Define
\begin{equation}
	G_{ip}:=\sum_{\ell=1}^{n_x}m_{ip,\ell}H_{i\ell}^{\rm s},
	\qquad A_{ip}:=S_iG_{ip}S_i .
	\label{eq:pf_row_quadratic}
\end{equation}
For each \(p=1,\ldots,n_x\) and face sign
\(\varepsilon\in\{-1,+1\}\), a fixed nonnegative vector
\(d_{ip}^{\varepsilon}\) is used only after verifying
\begin{equation}
	\operatorname{Diag}(d_{ip}^{\varepsilon})+\varepsilon A_{ip}\succeq0.
	\label{eq:signed_dm_condition}
\end{equation}
Consequently,
\(-\varepsilon y^{\mathsf T}G_{ip}y
\le\sum_{q=1}^{n_x}d_{ipq}^{\varepsilon}\chi_{iq}\) throughout
\(\mathcal Y_i(\beta_i)\).  Using the exact support function of the weighted
\(\ell_1\) ball, the two signed self-mapping conditions are
\begin{equation}
	\begin{aligned}
		&-\varepsilon m_{ip}^{\mathsf T}(F_i^0+B_i\Delta u)
		+r_i\bigl\|D^{-\mathsf T}C_i^{\mathsf T}m_{ip}\bigr\|_\infty \\
		&\hspace{16mm}
		+\sum_{q=1}^{n_x}d_{ipq}^{\varepsilon}\chi_{iq}
		\le \varsigma_{ip}\beta_{ip},
		\quad \varepsilon\in\{-1,+1\}.
	\end{aligned}
	\label{eq:component_selfmap}
\end{equation}
They guarantee
\(\Phi_i(\mathcal Y_i(\beta_i);u,\delta)
\subseteq\mathcal Y_i(\beta_i)\) for every
\(\|D\delta\|_1\le r_i\).
To establish uniqueness and Jacobian regularity, define the nonnegative
comparison matrix, for \(p,k=1,\ldots,n_x\),
\begin{equation}
	[\mathcal L_i(\beta_i)]_{pk}
	:=2\sum_{q=1}^{n_x}
	\bigl|[G_{ip}]_{kq}\bigr|\varsigma_{iq}\beta_{iq}.
	\label{eq:component_lipschitz}
\end{equation}
A fixed Perron scaling \(\omega_i\in\mathbb R_{++}^{n_x}\) and a numerical
margin \(\eta_J\in(0,1)\) give the affine contraction condition
\begin{equation}
	\mathcal L_i(\beta_i)\omega_i
	\le(1-\eta_J)\omega_i
	\quad\text{componentwise}.
	\label{eq:perron_contraction}
\end{equation}

\begin{theorem}[Sample-ball regular-root certificate]
	\label{thm:pf_branch_ball}
	Suppose that \eqref{eq:pf_exact_quadratic} holds and \(J_i\) is
	nonsingular.  If
	\eqref{eq:component_tube}--\eqref{eq:perron_contraction} hold, then, for every
	\(\|D\delta\|_1\le r_i\), there exists a unique solution
	\begin{equation}
		x_{\rm PF}(u,\widehat\xi_i+\delta)
		\in\bar x_i+\mathcal Y_i(\beta_i).
		\label{eq:pf_solution_in_tube}
	\end{equation}
	The solution depends continuously on \(\delta\), and \(J_xF\) is
	nonsingular throughout the certified tube.  Let \(x_i^c\) denote the exact
	root at \(\delta=0\).  If \(x_i^c\) has been identified with the target PF
	branch at initialization or by Proposition~\ref{prop:branch_label}, then
	the roots throughout the sample ball inherit this label.
	\end{theorem}

\begin{proof}
	The verified matrix inequality in \eqref{eq:signed_dm_condition}, the square
	epigraphs,
	and
	\(
	\sup_{\|D\delta\|_1\le r_i}a^{\mathsf T}\delta
	=r_i\|D^{-\mathsf T}a\|_\infty
	\)
	show that \eqref{eq:component_selfmap} is a componentwise self-mapping
	certificate.

	Quadratic polarization and \eqref{eq:component_lipschitz} give
	\(|\Phi_i(y)-\Phi_i(y')|
	\le\mathcal L_i(\beta_i)|y-y'|\).  Let
	\(\|v\|_{\omega_i,\infty}:=\max_p|v_p|/\omega_{ip}\).  Condition
	\eqref{eq:perron_contraction} makes \(\Phi_i\) a strict contraction in this
	norm.  Banach's fixed-point theorem gives existence, uniqueness, and
	continuous parameter dependence.
	Finally,
	\begin{equation}
		I-D_y\Phi_i(y)=M_iJ_xF(\bar x_i+y,u,\widehat\xi_i+\delta),
		\label{eq:pf_jacobian_identity}
	\end{equation}
	and the left-hand side is nonsingular by the Neumann lemma.  Since \(M_i\)
	is nonsingular, so is \(J_xF\).  For any admissible \(\delta\), the path
	\(\gamma_\delta\mapsto\gamma_\delta\delta\), \(\gamma_\delta\in[0,1]\),
	remains in the convex sample ball.  Continuous
	parameter dependence therefore connects its unique roots to \(x_i^c\)
	without encountering a singularity, which proves the conditional
	branch-label statement.
\end{proof} 
Because \(F_i^0\) is retained, \(\bar x_i\) is a numerical center rather than
an assumed exact solution. The unique exact PF root \(x_i^c\), enclosed by
the verified isolation tube around \(\bar x_i\), belongs to the target branch,
whose identity is propagated between accepted iterates as follows.
\subsection{Exact Adjoint Safety-Matrix Identity}
\label{subsec:adjoint_identity}

Each lift coordinate used by the stability and operating-limit blocks has
the exact form
\begin{equation}
	z_a(x,u,\xi)
	=x^{\mathsf T}Z_ax+p_a^{\mathsf T}x
	+d_a^{\mathsf T}u+e_a^{\mathsf T}\xi+f_a .
	\label{eq:safety_output_quadratic}
\end{equation}
Here, \(Z_a\in\mathbb R^{n_x\times n_x}\),
\(p_a\in\mathbb R^{n_x}\), \(d_a\in\mathbb R^{n_u}\),
\(e_a\in\mathbb R^m\), and \(f_a\in\mathbb R\),
for \(a=1,\ldots,n_z\).
Case-specific linear coordinates are included by setting \(Z_a=0\); in the
Iva specialization, this covers the \(\upsilon_b\) terms in
\eqref{eq:iva_pencil} without approximation.
Writing \(Z_a^{\rm s}:=(Z_a+Z_a^{\mathsf T})/2\),
\(g_{ia}=2Z_a^{\rm s}\bar x_i+p_a\), and
\(\bar z_{ia}=z_a(\bar x_i,\bar u,\widehat\xi_i)\) gives
\begin{equation}
	z_a
	=\bar z_{ia}+g_{ia}^{\mathsf T}y
	+y^{\mathsf T}Z_a^{\rm s}y
	+d_a^{\mathsf T}\Delta u+e_a^{\mathsf T}\delta .
	\label{eq:safety_output_expansion}
\end{equation}
To eliminate the linear state term without approximating the PF graph,
solve the following adjoint system for
\(\lambda_{ia}\in\mathbb R^{n_x}\):
\begin{equation}
	J_i^{\mathsf T}\lambda_{ia}=g_{ia}.
	\label{eq:adjoint_system}
\end{equation}

\begin{lemma}[Exact adjoint safety-output identity]
	\label{lem:adjoint_identity}
	On the PF graph \eqref{eq:pf_exact_quadratic},
	\begin{equation}
		\boxed{\;
			z_a
			=z_{ia}^0
			+\widehat d_{ia}^{\mathsf T}\Delta u
			+\widehat e_{ia}^{\mathsf T}\delta
			+y^{\mathsf T}R_{ia}y
			\;}
		\label{eq:adjoint_identity}
	\end{equation}
	holds exactly, where
	\begin{subequations}\label{eq:adjoint_coefficients}
		\begin{align}
			z_{ia}^0&:=\bar z_{ia}-\lambda_{ia}^{\mathsf T}F_i^0,
			\label{eq:adjoint_z0}\\
			\widehat d_{ia}&:=d_a-B_i^{\mathsf T}\lambda_{ia},\qquad
			\widehat e_{ia}:=e_a-C_i^{\mathsf T}\lambda_{ia},
			\label{eq:adjoint_affine_coefficients}\\
			R_{ia}&:=Z_a^{\rm s}
			-\sum_{\ell=1}^{n_x}\lambda_{ia,\ell}H_{i\ell}^{\rm s}.
			\label{eq:adjoint_remainder_matrix}
		\end{align}
	\end{subequations}
\end{lemma}

\begin{proof}
	Equation \eqref{eq:adjoint_system} gives
	\(g_{ia}^{\mathsf T}y=\lambda_{ia}^{\mathsf T}J_i y\).  Substituting
	$$J_i y=-F_i^0-B_i\Delta u-C_i\delta-Q_i(y)$$ from
	\eqref{eq:pf_exact_quadratic} into \eqref{eq:safety_output_expansion} yields
	\eqref{eq:adjoint_identity}--\eqref{eq:adjoint_coefficients}.
\end{proof}

Thus the adjoint step does not replace or linearize the nonlinear PF
equations.  It removes the first-order state dependence only after the PF
equality is imposed, leaving an exact quadratic residual.  Only the distinct
lift directions appearing in the safety blocks require adjoint
right-hand sides.

To treat stability and operating limits uniformly, index their affine
symmetric blocks by \(\kappa\in\mathcal K_{\rm saf}\), and let \(n_\kappa\)
denote the order of block \(\kappa\):
\begin{equation}
	\mathsf M_\kappa(z,u)
	=\mathsf M_{\kappa0}+\mathsf M_{\kappa u}(u)
	+\sum_{a=1}^{n_z}z_a\mathsf B_{\kappa a}\succeq0.
	\label{eq:safety_block_pencil}
\end{equation}
Here, \(\mathsf M_{\kappa0}\), \(\mathsf M_{\kappa u}(u)\), and
\(\mathsf B_{\kappa a}\) belong to \(\mathbb S^{n_\kappa}\).
The stability block is
\(K^{\rm case}(z,u)-\tau_KI\); a scalar upper limit is represented by a
\(1\times1\) block after its strictness margin is absorbed; and an SOC limit
is represented by its standard arrow-PSD block.  These are assembled
blockwise in the implementation and are not formed as one dense matrix.
Let \(\widetilde\delta=D\delta\), and let \(\widehat E_i\) have row
\(a\) equal to \(\widehat e_{ia}^{\mathsf T}\).  Lemma~\ref{lem:adjoint_identity}
gives the exact block identity
\begin{equation}
	\mathsf M_\kappa
	=\mathsf M_{i\kappa}^{\rm c}(u)
	+\sum_{j=1}^{m}\widetilde\delta_j\mathsf M_{i\kappa j}
	+\mathcal R_{i\kappa}(y),
	\label{eq:adjoint_block_identity}
\end{equation}
where
\begin{subequations}\label{eq:adjoint_block_coefficients}
	\begin{align}
		\mathsf M_{i\kappa}^{\rm c}(u)
		&:=\mathsf M_{\kappa0}+\mathsf M_{\kappa u}(u)
		+\sum_{a=1}^{n_z}\bigl(z_{ia}^0
		+\widehat d_{ia}^{\mathsf T}\Delta u\bigr)\mathsf B_{\kappa a},
		\label{eq:block_center}\\
		\mathsf M_{i\kappa j}
		&:=\sum_{a=1}^{n_z}[\widehat E_iD^{-1}]_{aj}\mathsf B_{\kappa a},
		\label{eq:block_uncertainty_direction}\\
		\mathcal R_{i\kappa}(y)
		&:=\sum_{a=1}^{n_z}
		\bigl(y^{\mathsf T}R_{ia}y\bigr)\mathsf B_{\kappa a}.
		\label{eq:block_matrix_remainder}
	\end{align}
\end{subequations}
The nonlinear term in \eqref{eq:block_matrix_remainder} is an
operator-valued quadratic remainder.  Combining the output directions
before bounding this term retains cancellations that are lost when every
lifted coordinate is enclosed independently.

\subsection{Finite Matrix-Remainder Robust Containment}
\label{subsec:matrix_remainder}

We next construct a finite, verified lower bound for
\(\mathcal R_{i\kappa}(y)\).  First, fixed signed output majorants
\(d_{ia}^{\rm U},d_{ia}^{\rm L}\ge0\) satisfy
\begin{equation}
	\operatorname{Diag}(d_{ia}^{\rm U})-S_iR_{ia}S_i\succeq0,
	\qquad
	\operatorname{Diag}(d_{ia}^{\rm L})+S_iR_{ia}S_i\succeq0.
	\label{eq:output_dm_majorants}
\end{equation}
Hence, throughout \(\mathcal Y_i(\beta_i)\),
\begin{equation}
	-L_{ia}(\chi_i)
	\le y^{\mathsf T}R_{ia}y
	\le U_{ia}(\chi_i),
	\quad
	\begin{cases}
		U_{ia}(\chi_i)=(d_{ia}^{\rm U})^{\mathsf T}\chi_i,\\
		L_{ia}(\chi_i)=(d_{ia}^{\rm L})^{\mathsf T}\chi_i.
	\end{cases}
	\label{eq:signed_output_bounds}
\end{equation}
These one-sided bounds preserve residual asymmetry and provide a valid
coordinatewise enclosure.
For a matrix-level bound, define, for \(p,q=1,\ldots,n_x\),
\begin{equation}
	\mathsf C_{i\kappa,pq}
	:=\varsigma_{ip}\varsigma_{iq}
	\sum_{a=1}^{n_z}[R_{ia}]_{pq}\mathsf B_{\kappa a}.
	\label{eq:matrix_remainder_coefficients}
\end{equation}
For \(y=S_i\widetilde y\),
\begin{equation}
	\mathcal R_{i\kappa}(S_i\widetilde y)
	=\sum_p \widetilde y_p^2\mathsf C_{i\kappa,pp}
	+2\sum_{p<q}\widetilde y_p\widetilde y_q\mathsf C_{i\kappa,pq}.
	\label{eq:matrix_remainder_polynomial}
\end{equation}
Choose a fixed state-interaction core graph
\(\mathcal G_{i\kappa}^{\rm c}
=(\{1,\ldots,n_x\},\mathcal E_{i\kappa}^{\rm c})\).  Let
\(\mathbf C_{i\kappa}^{\rm c}\in\mathbb S^{n_xn_\kappa}\) be the block
matrix whose diagonal block \(p\) is
\(\mathsf C_{i\kappa,pp}\), whose \((p,q)\) block is
\(\mathsf C_{i\kappa,pq}\) for a retained edge, and whose remaining
off-diagonal blocks are zero.  Fixed matrices
\(P_{i\kappa p}\succeq0\) are accepted only if
\begin{equation}
	\mathbf C_{i\kappa}^{\rm c}
	+\operatorname{blkdiag}(P_{i\kappa1},\ldots,P_{i\kappa n_x})
	\succeq0.
	\label{eq:bdm_condition}
\end{equation}
The PSD block matrix in \eqref{eq:bdm_condition} is a degree-two Gram
certificate for a matrix polynomial, within the matrix sum-of-squares
framework of \cite{scherer2006matrixsos}.  The retained core and analytically
bounded tail exploit the shared quadratic state factors of the present
safety blocks.
For a symmetric matrix \(A\), define
\(|A|_{\rm m}:=(A^2)^{1/2}\).  The component matrix atoms are
\begin{equation}
	\Gamma_{i\kappa p}
	:=P_{i\kappa p}
	+\sum_{\substack{q=1\\q\ne p,\ \{p,q\}\notin
	\mathcal E_{i\kappa}^{\rm c}}}^{n_x}
	|\mathsf C_{i\kappa,pq}|_{\rm m}\succeq0,
	\label{eq:bdm_atoms}
\end{equation}
where each omitted undirected edge contributes once to each endpoint.
Congruence of \eqref{eq:bdm_condition} by
\(\operatorname{blkdiag}(\widetilde y_pI_{n_\kappa})\), together with
\(2\widetilde y_p\widetilde y_qA\succeq
-(\widetilde y_p^2+\widetilde y_q^2)|A|_{\rm m}\), yields
\begin{equation}
	\mathcal R_{i\kappa}(y)
	\succeq-\sum_{p=1}^{n_x}\chi_{ip}\Gamma_{i\kappa p}.
	\label{eq:component_matrix_remainder_bound}
\end{equation}
One closed-form feasible choice is
\begin{equation}
	\begin{aligned}
		P_{i\kappa p}^{\rm cf}
		&=(\mathsf C_{i\kappa,pp})_-
		+\sum_{\substack{q=1\\q\ne p,\ \{p,q\}\in
		\mathcal E_{i\kappa}^{\rm c}}}^{n_x}
		|\mathsf C_{i\kappa,pq}|_{\rm m},\\
		A_-&:=\frac{|A|_{\rm m}-A}{2}.
	\end{aligned}
	\label{eq:bdm_closed_form}
\end{equation}
satisfies \eqref{eq:bdm_condition}.  A small offline SDP may reduce these
matrices, but only a numerically verified feasible solution is retained for
the online problem.
For each sample--block pair, fix nonnegative split coefficients
\begin{equation}
	\lambda_{i\kappa}^{\rm box}
	+\lambda_{i\kappa}^{\rm mat}=1.
	\label{eq:remainder_split}
\end{equation}
They are fixed offline parameters rather than online variables; otherwise
products with \(\chi_i\) would be bilinear.  Introduce
\(X_{i\kappa a}\in\mathbb S^{n_\kappa}\), \(a=1,\ldots,n_z\), and define
\begin{subequations}\label{eq:signed_box_parameters}
	\begin{align}
		\mu_{i\kappa a}
		&:=\frac{\lambda_{i\kappa}^{\rm box}}{2}
		\bigl(U_{ia}(\chi_i)-L_{ia}(\chi_i)\bigr),\\
		\nu_{i\kappa a}
		&:=\frac{\lambda_{i\kappa}^{\rm box}}{2}
		\bigl(U_{ia}(\chi_i)+L_{ia}(\chi_i)\bigr),\\
		X_{i\kappa a}&\succeq
		\nu_{i\kappa a}\mathsf B_{\kappa a},
		\qquad
		X_{i\kappa a}\succeq
		-\nu_{i\kappa a}\mathsf B_{\kappa a},
		\label{eq:matrix_cube_epigraph}\\
		T_{i\kappa}(\chi_i)
		&:=\lambda_{i\kappa}^{\rm mat}
		\sum_{p=1}^{n_x}\chi_{ip}\Gamma_{i\kappa p}.
		\label{eq:matrix_remainder_allocation}
	\end{align}
\end{subequations}
Together with the shifted center terms, the domination LMIs in
\eqref{eq:matrix_cube_epigraph} give the Matrix-Cube safe counterpart for the
asymmetric interval box \cite{bental2002matrixcube}.
The coordinate-box part of the remainder is then bounded below by
\(\sum_{a=1}^{n_z}\mu_{i\kappa a}\mathsf B_{\kappa a}
-\sum_{a=1}^{n_z}X_{i\kappa a}\), while
\eqref{eq:component_matrix_remainder_bound} bounds the matrix-shaped part by
\(-T_{i\kappa}(\chi_i)\).  Therefore the finite robust counterpart is
\begin{equation}
	\boxed{
		\begin{aligned}
			\mathsf M_{i\kappa}^{\rm c}(u)
			+\varepsilon r_i\mathsf M_{i\kappa j}
			+\sum_{a=1}^{n_z}
			\bigl(\mu_{i\kappa a}\mathsf B_{\kappa a}-X_{i\kappa a}\bigr)
			&\succeq T_{i\kappa}(\chi_i),\\
			\varepsilon\in\{-1,+1\},\quad j=1,\ldots,m,
			&\quad \kappa\in\mathcal K_{\rm saf}.
	\end{aligned}}
	\label{eq:finite_vertex_lmis}
\end{equation}
Indeed,
\(\{\widetilde\delta:\|\widetilde\delta\|_1\le r_i\}
=\operatorname{conv}\{\pm r_i\mathbf e_j\}_{j=1}^m\), where
\(\mathbf e_j\) is the \(j\)-th standard basis vector,
so the affine uncertainty dependence is handled exactly by the \(2m\)
vertices.  The choices
\((\lambda^{\rm box},\lambda^{\rm mat})=(1,0)\) and \((0,1)\) recover the
signed coordinate-box and pure matrix-remainder certificates,
respectively.  With these offline quantities fixed,
\eqref{eq:component_selfmap}, \eqref{eq:perron_contraction},
\eqref{eq:component_square_epigraph}, and
\eqref{eq:signed_box_parameters}--\eqref{eq:finite_vertex_lmis} are linear,
SOC, or LMI constraints in the online variables.  No semi-infinite
constraint remains.

\subsection{Certified-Radius Guarantee}
\label{subsec:radius_guarantee}

\begin{theorem}[Certified sample-safe radius]
	\label{thm:sample_safe_radius}
	Fix sample \(i\).  Suppose that the declared stability-certificate theorem
	and all of its applicability conditions hold uniformly on the certified
	domain.  If the hypotheses of Theorem~\ref{thm:pf_branch_ball} hold, its
	exact center root carries the target-branch label, the adjoint solves and
	all numerical bounds have been verified, and
	\eqref{eq:output_dm_majorants}--\eqref{eq:finite_vertex_lmis} hold for every
	stability and operating block, then
	\begin{equation}
		\mathcal B_i(r_i)
		\subseteq\mathcal S^{\rm cert}(u)
		\subseteq\mathcal S^{\rm phys}(u).
		\label{eq:certified_ball_chain}
	\end{equation}
	Consequently,
	\begin{equation}
		\boxed{\quad
			0\le r_i\le d_i^{\rm cert}(u)\le d_i^{\rm phys}(u).
			\quad}
		\label{eq:certified_radius_order}
	\end{equation}
\end{theorem}

\begin{proof}
	Take any \(\xi=\widehat\xi_i+\delta\in\mathcal B_i(r_i)\).  By
	Theorem~\ref{thm:pf_branch_ball}, the target PF branch has a unique
	regular equilibrium with \(y\in\mathcal Y_i(\beta_i)\).  Lemma
	\ref{lem:adjoint_identity} and \eqref{eq:adjoint_block_identity} give the
	exact value of every safety block at that equilibrium.  Equations
	\eqref{eq:signed_output_bounds}--\eqref{eq:matrix_remainder_allocation} provide a valid lower
	bound for the full quadratic matrix remainder.  Since
	\(\widetilde\delta=D\delta\) lies in the convex hull of the \(2m\)
	vertices, the LMIs in
	\eqref{eq:finite_vertex_lmis} imply
	\(\mathsf M_\kappa(z,u)\succeq0\) for every
	\(\kappa\in\mathcal K_{\rm saf}\).  Thus all buffered operating limits and
	the fixed lifted-PSD stability certificate hold, proving the first
	inclusion in \eqref{eq:certified_ball_chain}.  The second inclusion follows
	from the model-specific certificate theorem.  A closed ball contained in
	the certified safe set cannot intersect its closed failure set, which gives
	\eqref{eq:certified_radius_order}.
\end{proof}

The rectangular-coordinate PF expansion, the adjoint identity on the PF
graph, and the \(2m\)-vertex reduction of the weighted-\(\ell_1\) affine
uncertainty are exact.  Conservatism enters through the component PF tube,
the finite matrix-remainder certificate, and, when the stability pencil
provides only a sufficient condition, the gap between
\(\mathcal S^{\rm cert}\) and \(\mathcal S^{\rm phys}\).  The contribution of
Theorem~\ref{thm:sample_safe_radius} is their composition into a certified
lower bound \(r_i\) on the sample-to-failure distance required by the
Wasserstein risk interface of Section~\ref{sec:risk_interface}.

\section{Sequential Certified WDRO-OPF}
\label{sec:wasserstein_master}

\subsection{Convex Master With Fixed Local Bounds}
\label{subsec:sequential_master}

At iteration \(k\), the local construction of Section~\ref{sec:certified_radii}
uses \(\bar u=u^{(k)}\) and \(\bar x_i=\bar x_i^{(k)}\) for
\(i\in\{0\}\cup[N]\), where \(i=0\) denotes the nominal center and
\(\widehat\xi_0:=\xi^{\rm nom}\).  Its bounds are
recomputed and fixed.  Let
\(\mathcal C_i^{(k)}(u,r_i,\zeta_i)\), \(i\in[N]\), denote the conic
conditions for \(\mathcal B_i(r_i)\subseteq\mathcal S^{\rm cert}(u)\),
where \(\zeta_i\) collects their auxiliary variables, and let
\(\mathcal C_0^{(k)}(u,\zeta_0)\) denote the nominal block, with \(\zeta_0\)
its auxiliary variables.  Write
\(\zeta=\operatorname{col}(\zeta_0,\ldots,\zeta_N)\).  Let
\(W_u\in\mathbb R^{n_u\times n_u}\) be a fixed
nonsingular dispatch scaling and \(\Delta_k>0\) the trust-region radius.  Define
\(\mathcal U^{(k)}:=\{u\in\mathcal U:
\|W_u(u-u^{(k)})\|_\infty\le\Delta_k\}\), and let \(\bar r_i^{(k)}\) be the
corresponding radius limit.  The iteration-$k$ master is
\begin{equation}
	\begin{aligned}
		\min_{u,r,t,s,\zeta}\quad & C(u)\\
		\mathrm{s.t.}\quad
		&u\in\mathcal U^{(k)},\quad \mathcal C_0^{(k)}(u,\zeta_0),\\
		&\mathcal C_i^{(k)}(u,r_i,\zeta_i),\quad
		0\le r_i\le\bar r_i^{(k)},
		&&i\in[N],\\
		&r_i\ge t-s_i,\quad s_i\ge0,
		&&i\in[N],\\
		&\rho+N^{-1}\textstyle\sum_i s_i\le\alpha t,
		\quad t\ge0.
	\end{aligned}
	\label{eq:sequential_certified_master}
\end{equation}
For fixed bounds and convex $C$ and $\mathcal U$, this is an SDP/SOCP.  Since
the PF anchors and bounds are updated between accepted iterates, the overall
procedure is sequential and does not provide a global optimum of the original
nonlinear WDRO-OPF.
For each \(i\in\{0\}\cup[N]\), let
\(\mathcal X_{i,\rm root}^{(k)}\) be a verified enclosure of the exact root
carrying the target-branch label at the accepted dispatch \(u^{(k)}\).  For a
candidate \(u^+\), set \(\Delta u^+=u^+-u^{(k)}\).  Let
\((\beta_i^0,\chi_i^0)\) define a strictly feasible starting PF tube and let
\((\beta_i^+,\chi_i^+)\) be the candidate tube variables evaluated at
zero uncertainty radius.  Using the same frozen PF bounds, define
\begin{equation}
	\begin{gathered}
		u(\gamma)=u^{(k)}+\gamma\Delta u^+,\qquad
		\beta_i(\gamma)=(1-\gamma)\beta_i^0+\gamma\beta_i^+,\\[-1mm]
		\chi_i(\gamma)=(1-\gamma)\chi_i^0+\gamma\chi_i^+,
		\qquad 0\le\gamma\le1,
	\end{gathered}
	\label{eq:pf_path_lift}
\end{equation}
with the remaining PF epigraph variables interpolated in the same way.

\begin{proposition}[PF branch-label propagation]
	\label{prop:branch_label}
	Suppose that the smooth PF equations, coordinates, discrete operating mode,
	preconditioner, and frozen local bounds remain unchanged along
	\eqref{eq:pf_path_lift}.  Assume that
	\(\mathcal X_{i,\rm root}^{(k)}
	\subseteq\bar x_i^{(k)}+\mathcal Y_i(\beta_i^0)\) and that both
	endpoint PF conditions satisfy the square epigraphs and have strictly
	positive self-mapping and Perron-contraction margins and positive tube
	widths.  Then the unique roots in the interpolated tubes form a
	continuous, locally smooth, and regular PF path from the labelled root to
	the candidate endpoint.  At \(u^+\), construct from an independent PF
	solve a verified root enclosure with half-width
	\(\Delta x_{i,\rm root}^+\in\mathbb R_+^{n_x}\),
	\(\mathcal X_{i,\rm root}^{+}=\{x:
	|x-\bar x_i^+|\le \Delta x_{i,\rm root}^+\}\).
	Let \(S_i^{(k)}=\operatorname{Diag}(\varsigma_i^{(k)})\) be the fixed scaling
	used in \(\mathcal Y_i\) at iteration \(k\).  If, componentwise,
	\begin{equation}
		|\bar x_i^+-\bar x_i^{(k)}|+\Delta x_{i,\rm root}^+
		\le S_i^{(k)}\beta_i^+,
		\label{eq:pf_endpoint_handoff}
	\end{equation}
	then its exact root is the endpoint of that path and inherits the
	target-branch label.
\end{proposition}

\begin{proof}
	The frozen self-mapping and contraction inequalities are affine in the
	interpolated quantities, while their epigraph constraints are convex;
	feasibility and the strict contraction margins therefore hold for every
	\(\gamma\in[0,1]\).  Banach's theorem and
	\eqref{eq:pf_jacobian_identity} give one regular root in each tube.
	Strict interiority and the implicit-function theorem join these roots into
	the stated path.
	Condition~\eqref{eq:pf_endpoint_handoff} places the new root enclosure in
	the old endpoint tube, whose root is unique, so the two roots coincide.
\end{proof}
This proposition certifies PF branch identity only; it does not assert that
the dispatch path satisfies the operating or stability constraints.

\subsection{Acceptance and Distributional Guarantee}
\label{subsec:computational_procedure}

At iteration $k$, verified local bounds are constructed from PF solutions at
$u^{(k)}$, and \eqref{eq:sequential_certified_master} is solved.  A candidate
is accepted only if independent PF solves verify all safety and Wasserstein
conditions and Proposition~\ref{prop:branch_label} holds for the nominal and
sample centers; otherwise the trust region is reduced and the bounds are
recomputed.  Verified continuation is used if the direct path or endpoint
handoff is inconclusive.  The algorithm returns the last accepted iterate
after two consecutive accepted iterates meet the prescribed tolerances with
positive certified margins.

   \section{Case Studies}
  \label{sec:case_studies}
  
  \subsection{Mechanism Study: Certified Radius Versus First Failure}
  \label{subsec:mechanism_radius}
  
  We use a two-bus lossless GFM system to quantify how closely the
  proposed certified radius approaches the first target-model failure
  distance.  Both voltage magnitudes are fixed at \(1\) p.u., and renewable
  and load forecast errors perturb the active-power transfer across the
  single line.  The uncertainty is measured in the same weighted
  \(\ell_1\) metric used by the Wasserstein ambiguity set.
 
  Let \(q=D\delta\).  The line transfer can be written as
  \begin{equation}
  	P(q)=P_0+\bm h^{\mathsf T}q,
  	\qquad
  	P_0=0.6,\qquad
  	\bm h=(-0.05,\ 0.05)^{\mathsf T}.
  	\label{eq:mechanism_transfer}
  \end{equation}
  For this two-bus high-voltage branch,
  \begin{equation}
  	P=b\sin\phi,
  	\label{eq:mechanism_pf}
  \end{equation}
  and the projected Iva matrix is positive definite precisely when
  \begin{equation}
  	\cos\phi>
  	\frac{2b}{1/\beta_q+2b}.
  	\label{eq:mechanism_iva_condition}
  \end{equation}
  The reactive-droop coefficient is selected so that the critical transfer is
  \(P_{\rm crit}=0.602\) p.u.  Hence the nearest weighted-\(\ell_1\)
  failure distance is available in closed form:
  \begin{equation}
  	\begin{aligned}
  		d_{\rm first}
  		&=
  		\min_{\bm h^{\mathsf T}q=P_{\rm crit}-P_0}
  		\|q\|_1                                                   \\
  		&=
  		\frac{P_{\rm crit}-P_0}{\|\bm h\|_\infty}
  		=
  		\frac{0.002}{0.05}
  		=0.04 .
  	\end{aligned}
  	\label{eq:mechanism_exact_distance}
  \end{equation}
  The opposite transfer direction reaches its boundary at a larger distance.
  The PF Jacobian remains regular and the operating limits retain positive
  margins at \(d_{\rm first}\); therefore, the first event is the loss of
  small-signal stability of the registered GFM model.
  
  Table~\ref{tab:mechanism_radius} compares this distance with the radius
  returned by the proposed P1 certificate.  The certified radius is
  \(0.03993374\), which captures \(99.8344\%\) of the exact distance.  Its
  relative conservatism is only \(0.1656\%\).
  
  \begin{table}[t]
  	\centering
  	\caption{Certified Radius Versus First Target-Model Failure Distance}
  	\label{tab:mechanism_radius}
  	\begin{tabular}{lc}
  		\toprule
  		Quantity & Value \\
  		\midrule
  		Stress fraction & \(0.995\) \\
  		Certified radius \(r_{\rm cert}\)
  		& \(0.0399337418\) \\
  		First failure distance \(d_{\rm first}\)
  		& \(0.0400000000\) \\
  		Absolute gap \(d_{\rm first}-r_{\rm cert}\)
  		& \(6.6258\times10^{-5}\) \\
  		Relative gap
  		& \(0.1656\%\) \\
  		Captured distance \(r_{\rm cert}/d_{\rm first}\)
  		& \(99.8344\%\) \\
  		\bottomrule
  	\end{tabular}
  \end{table}
  
  Fig.~\ref{fig:iva_mechanism} explains the corresponding stability
  mechanism.  The horizontal coordinate is the weighted-\(\ell_1\)
  displacement normalized by \(d_{\rm first}\).  As the normalized distance
  reaches one, the minimum eigenvalue of the Iva matrix and the spectral
  abscissa of the registered dynamic model cross zero at the same point.
  The certified radius lies immediately inside this boundary.
  
  \begin{figure*}[t]
  	\centering
  	\subfloat[Certificate margin and target-model spectral abscissa.]{
  		\includegraphics[width=0.485\textwidth]
  		{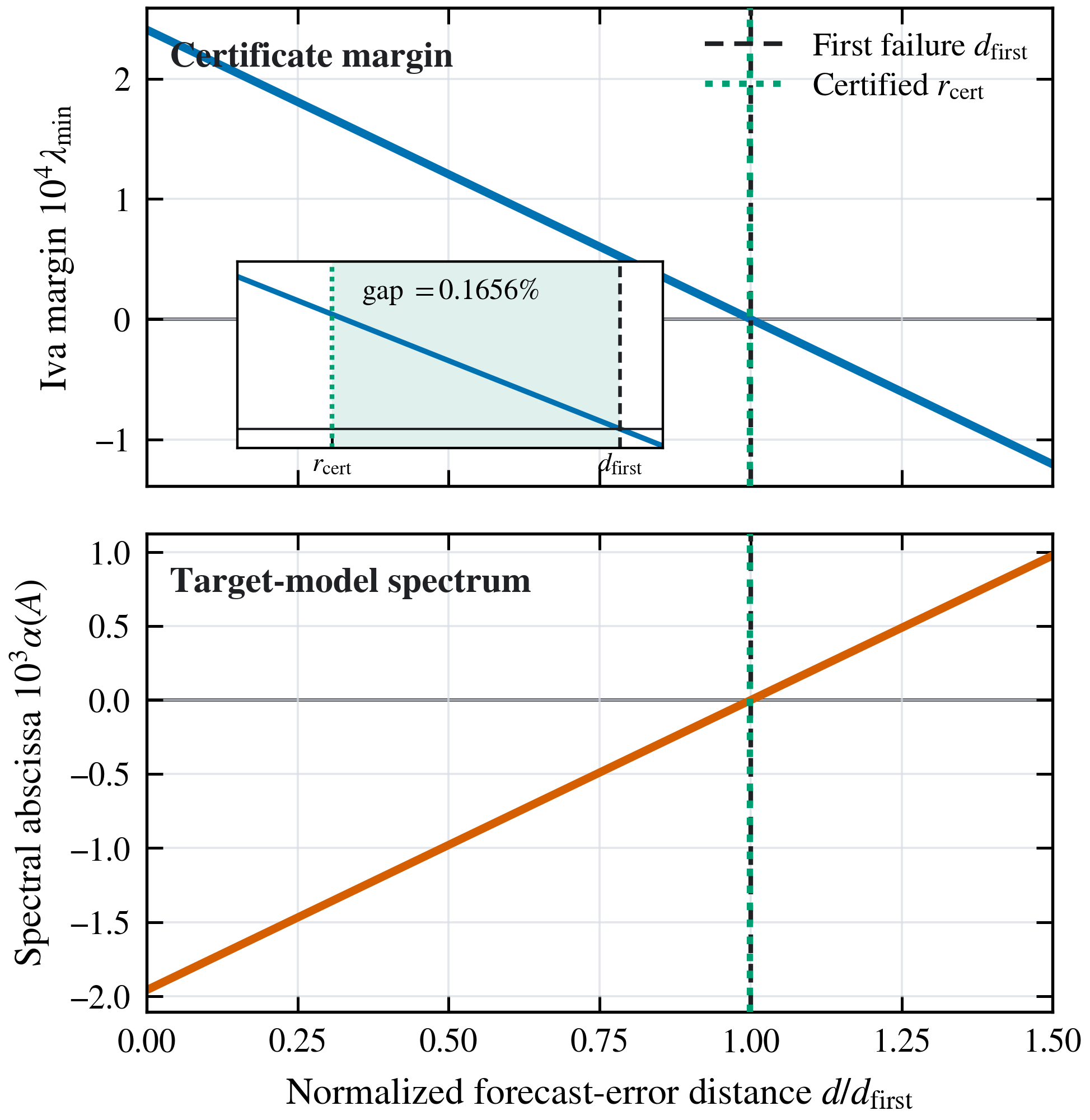}
  		\label{fig:iva_certificate_spectrum}
  	}
  	\hfill
  	\subfloat[Nonlinear trajectories on the stable and unstable sides.]{
  		\includegraphics[width=0.485\textwidth]
  		{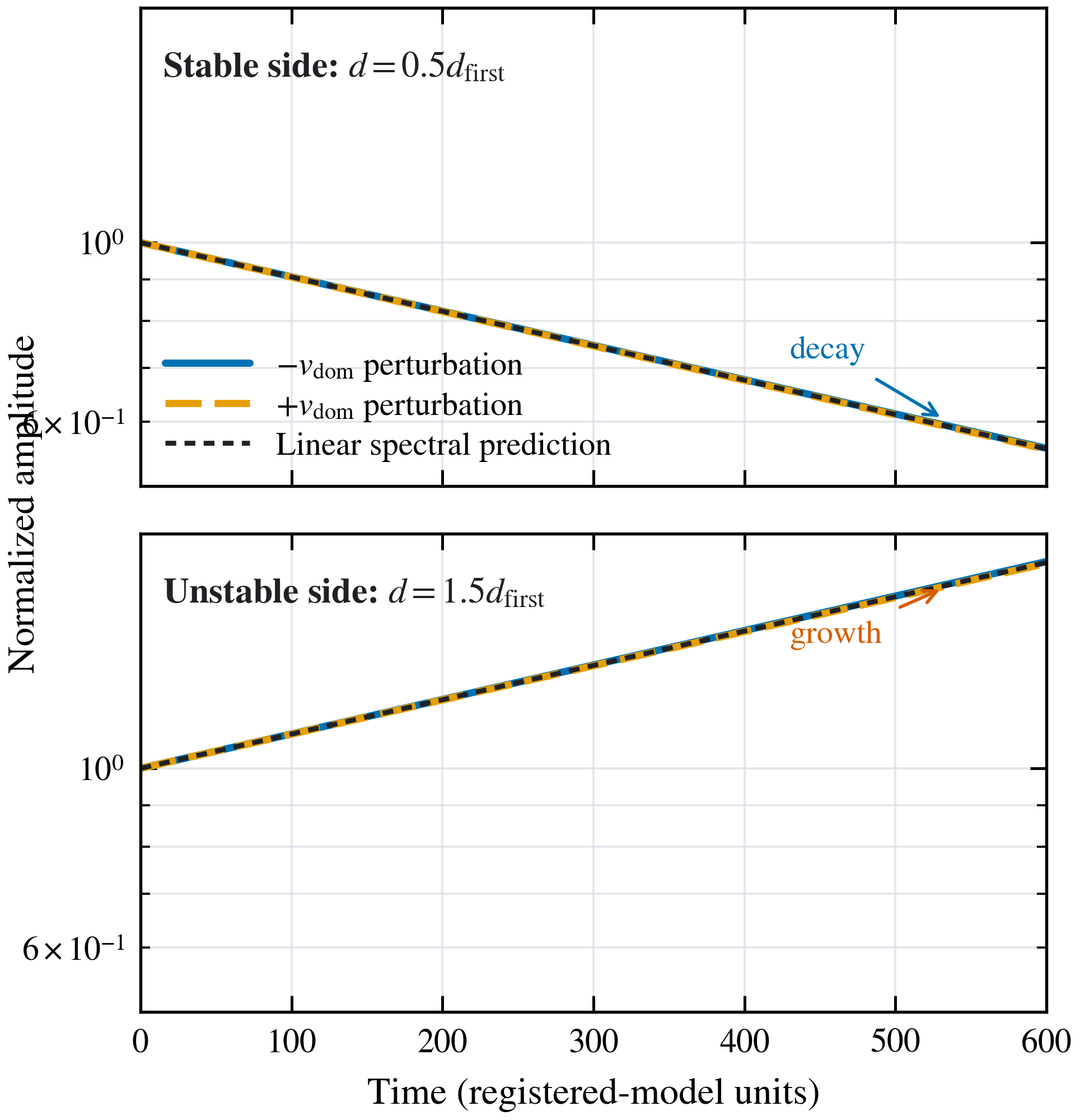}
  		\label{fig:iva_nonlinear_tds}
  	}
  	\caption{Mechanism study for the two-bus Iva--GFM system.
  		The dotted line marks the P1 certified radius and the dashed line marks
  		the closed-form first failure distance.  The Iva eigenvalue, target-model
  		spectrum, and nonlinear trajectories consistently identify the same
  		stability transition.}
  	\label{fig:iva_mechanism}
  \end{figure*}
  
  Immediately inside the boundary, at normalized distance \(0.9999\), the
  Iva minimum eigenvalue is \(2.41\times10^{-8}\) and the target-model
  spectral abscissa is \(-1.96\times10^{-7}\).  Immediately outside, at
  normalized distance \(1.0001\), the corresponding values are
  \(-2.41\times10^{-8}\) and \(1.96\times10^{-7}\).  At the crossing, the
  minimum singular value of the PF Jacobian is approximately \(0.694\),
  confirming that the stability boundary is reached while the equilibrium
  branch remains regular.
  
  The nonlinear simulations provide a complementary time-domain
  interpretation.  Perturbations initialized along both signs of the dominant
  mode decay on the stable side and grow on the unstable side.  The fitted
  growth rates agree with the linearized spectral abscissae within
  \(4.21\times10^{-6}\), with coefficients of determination above
  \(0.99999989\).  Together, these results show that the certified radius
  closely tracks the first small-signal stability failure while preserving a
  strict inner safety margin.
  All statements in this subsection refer to the registered reduced
  standard-droop GFM target model used by the Iva certificate.

\section{Conclusion}
\label{sec:conclusion}
This paper developed a certified framework that combines branch-preserving AC power-flow enclosures, lifted PSD stability containment, and exact Wasserstein sample-distance aggregation to obtain
tractable and rigorous distributionally robust small-signal stability guarantees.

\ifprintreferences
  \bibliographystyle{IEEEtran}
  \bibliography{references}
\fi

\end{document}